\documentclass[a4paper,UKenglish,cleveref, autoref, thm-restate]{lipics-v2021}
\graphicspath{{./Figures/}}%helpful if your graphic files are in another directory

\title{Bi-objective Search with Bi-directional A*} %TODO Please add

\titlerunning{Bi-objective Search with Bi-directional A*} %TODO optional, please use if title is longer than one line

\author{Saman Ahmadi\footnote{Corresponding Author}}{Department of Data Science and AI, Monash University, Australia \and CSIRO Data61, Australia}{saman.ahmadi@monash.edu}{https://orcid.org/0000-0002-7326-3384}{}
\author{Guido Tack}{Department of Data Science and AI, Monash University, Australia}{guido.tack@monash.edu}{}{}
\author{Daniel Harabor}{Department of Data Science and AI, Monash University, Australia}{daniel.harabor@monash.edu}{}{}
\author{Philip Kilby}{CSIRO Data61, Australia}{philip.kilby@data61.csiro.au}{}{}
\authorrunning{S. Ahmadi, G. Tack, D. Harabor, and P. Kilby}
\Copyright{Saman Ahmadi, Guido Tack, Daniel Harabor, and Philip Kilby}
\ccsdesc{Computing methodologies~Search methodologies}
\ccsdesc{Theory of computation~Shortest paths}
\keywords{Bi-objective search, heuristic search, shortest path problems} %TODO mandatory; please add comma-separated list of keywords

\category{} %optional, e.g. invited paper

\relatedversiondetails[]{Previous Version}{https://arxiv.org/abs/2105.11888} %linktext and cite are optional

\funding{Research at Monash University is supported by the Australian Research Council (ARC) under grant numbers DP190100013 and DP200100025 as well as a gift from Amazon.}
\nolinenumbers %uncomment to disable line numbering

\hideLIPIcs  %uncomment to remove references to LIPIcs series (logo, DOI, ...), e.g. when preparing a pre-final version to be uploaded to arXiv or another public repository

\EventEditors{Petra Mutzel, Rasmus Pagh, and Grzegorz Herman}
\EventNoEds{3}
\EventLongTitle{29th Annual European Symposium on Algorithms (ESA 2021)}
\EventShortTitle{ESA 2021}
\EventAcronym{ESA}
\EventYear{2021}
\EventDate{September 6--8, 2021}
\EventLocation{Lisbon, Portugal}
\EventLogo{}
\SeriesVolume{204}
\ArticleNo{16}
\usepackage{booktabs}
\usepackage{amsmath}  
\usepackage{amssymb}
\usepackage{amsthm}
\usepackage{pgfplots}
\pgfplotsset{compat=1.9}
\usepackage{floatrow}
\newfloatcommand{capbtabbox}{table}[]%[\FBwidth]
\usetikzlibrary{decorations.pathmorphing}
\usepackage[noend,ruled,linesnumbered,vlined]{algorithm2e}
\newcounter{algoline}
\newcommand{\nlast}{\refstepcounter{algoline}\nlset{\rlap{*}}}
\SetKwInput{KwInput}{Inputs}
\SetKwInput{KwOutput}{Output}
\usepackage{xcolor}
\SetNlSty{bfseries}{\color{black}}{}
\begin{document}

\maketitle

\begin{abstract}
Bi-objective search is a well-known algorithmic problem, concerned with finding a set of optimal solutions in a two-dimensional domain.
This problem has a wide variety of applications such as planning in transport systems or optimal control in energy systems.
Recently, bi-objective A*-based search (BOA*) has shown state-of-the-art performance in large networks.
This paper develops a bi-directional and parallel variant of BOA*, enriched with several speed-up heuristics.
Our experimental results on 1,000 benchmark cases show that our bi-directional A* algorithm for bi-objective search (BOBA*) can optimally solve all of the benchmark cases within the time limit, outperforming the state of the art BOA*, bi-objective Dijkstra and bi-directional bi-objective Dijkstra by an average runtime improvement of a factor of five over all of the benchmark instances.
\end{abstract}
\section{Introduction}
Bi-objective search aims at finding a set of non-dominated, Pareto-optimal solutions in a domain with two objectives \cite{deb2014multi}. It has a wide range of real-world applications, such as planning routes for maritime transportation based on both the fuel consumption and the total risk of the vehicle route \cite{DBLP:journals/cor/VenetiMKPV17}, or energy efficient paths for electric vehicles with arrival time considerations \cite{shen2019energy}.
When the underlying system is a network, the problem is finding a set of paths between two points that are not dominated by other solution paths.

A comparison of traditional approaches to the bi-objective one-to-all shortest path problem, such as the label correcting algorithm in \cite{DBLP:journals/cor/SkriverA00}, the label setting approach in \cite{guerriero2001label}, and the adaptation of a near shortest path procedure in \cite{DBLP:journals/networks/CarlyleW05}, was presented in \cite{comparison_RaithE09}.
These label-based approaches have been extended in several recent papers.
A generalisation of Dijkstra's algorithm and its bi-directional counterpart (for the one-to-one variant) to the bi-objective problem was presented in \cite{DBLP:journals/eor/Sedeno-NodaC19} by utilising the pruning strategies of \cite{DBLP:journals/eor/DuqueLM15} to avoid expanding unpromising paths during the search. 
The results show that the state-of-the-art bi-objective Dijkstra algorithm can outperform the bounded label setting approach in \cite{raith2010speed} and the depth-first search-based \textit{Pulse} algorithm in \cite{DBLP:journals/eor/DuqueLM15} on large-size instances.

Another recent work on point-to-point bi-objective search is the Bi-Objective A* search scheme (BOA*) in \cite{ulloa2020simple}.
BOA* is a standard A* heuristic search that leverages the fast dominance checking procedure of \cite{DBLP:journals/cor/PulidoMP15} for multi-objective search. 
In contrast to eager dominance checking approaches, as in \cite{DBLP:journals/eor/Sedeno-NodaC19},
BOA* lazily postpones dominance checking for newly generated nodes until their expansion.
The experimental results in \cite{ulloa2020simple} on a set of large instances show that the efficient dominance checking helps BOA* to perform better than the bi-objective Dijkstra algorithm of \cite{DBLP:journals/eor/Sedeno-NodaC19} and other best-first search approaches such as the label-setting multi-objective search NAMOA* of \cite{DBLP:journals/eswa/MachucaM12} and its improved version with a dimensionality reduction technique called NAMOA*\textsubscript{dr} \cite{DBLP:journals/cor/PulidoMP15}.

In this paper, we present Bi-Objective Bi-directional A* (BOBA*), a bi-directional extension of the BOA* algorithm that is easy to parallelise, uses different objective orders and includes several new heuristics to speed up the search. Our experiments on a set of 1,000 large test cases from the literature show that BOBA* can solve all of the cases to optimality, outperforming the state-of-the-art algorithms in both runtime and memory requirement.
\section{Background and Notation}
For a directed bi-objective graph $G=(S,E)$ with a finite set of states $S$ and a set of edges $E \subseteq S \times S$, the point-to-point bi-objective search problem is to find the set of Pareto-optimal solution paths from $\mathit{start} \in S$ to $\mathit{goal} \in S$ that are not dominated by any solution for both objectives.
Every edge $e \in E$ has two non-negative attributes accessed via the cost function ${\bf cost}:E \rightarrow \mathbb{R}^+ \times \mathbb{R}^+$.
A path is a sequence of states $s_i \in S$ with $ i \in \{1, \dots, n \}$.
The cost of path $p=\{s_1,s_2,s_3,\dots,s_n\}$ is then defined as the sum of corresponding attributes on all the edges constituting the path as ${\bf cost}(p) = \sum_{i=1}^{n-1}{{\bf cost}(s_i,s_{i+1})}$.
Following the standard notation in the heuristic search literature, we define our search objects to be \textit{nodes}. A node $\mathit{x}$ is a tuple that contains a state $s(x) \in S$; a value ${\bf g}(x)$ which measures the cost of a concrete path from the $\mathit{start}$ state to state $s(x)$; a value ${\bf f}(x)$ which is an estimate of the cost of a complete path from $\mathit{start}$ to $\mathit{goal}$ via $s(x)$; and a reference $\mathit{parent}(x)$ which indicates the parent of node $x$.
We perform a systematic search by \textit{expanding} nodes in best-first order. 
Each expansion operation \textit{generates} a set of successor nodes, each denoted $\mathit{Succ}(s(x))$, which are added into an \textit{Open} list. The \textit{Open} list sorts the nodes according to their {\bf f}-values in an ascending order, for the purpose of further expansion.

As with other A*-based algorithms, we compute {\bf f}-values using a consistent and admissible heuristic function ${\bf h}: S \rightarrow \mathbb{R}^+ \times \mathbb{R}^+$ \cite{hart1968formal}. In other words, ${\bf f}(x)={\bf g}(x)+{\bf h}(s)$ where ${\bf h}(s)$ is a lower bound on the cost of paths from state $s$ to $\mathit{goal}$. 
% The search is then led by the first best partial path which shows a lower cost estimate towards the $\mathit{goal}$ state.
Moreover, in bi-objective search, the cost function has two components which means that every (boldface) cost function is a tuple, eg. ${\bf f}=(f_1, f_2)$ or ${\bf h}=(h_1, h_2)$ and all operations are considered element-wise.
\begin{definition}
A heuristic function {\bf h} is consistent if we have ${\bf h}(s) \le {\bf cost}(s,t) + {\bf h}(t)$ for every edge $(s,t) \in E$. 
It is also admissible iff ${\bf h}(s) \le {\bf cost}(p)$ for every $s \in S$ and the optimal path \textit{p} from state $\mathit{s}$ to the $\mathit{goal}$ state. 
\end{definition}
\begin{definition}
For every pair of nodes ($x,y$) associated with the same state $s(x)=s(y)$, node $\mathit{y}$ is dominated by $\mathit{x}$ if we have $g_1(x) < g_1(y)$ and $g_2(x) \leq g_2(y)$ or if we have $g_1(x) = g_1(y)$ and $g_2(x) < g_2(y)$.
Node $\mathit{x}$ weakly dominates $\mathit{y}$ if $g_1(x) \leq g_1(y)$ and $g_2(x) \leq g_2(y)$.
\end{definition}
\textbf{Bi-objective A*:}
The Bi-Objective A* (BOA*) algorithm \cite{ulloa2020simple} first obtains its heuristic function \textbf{h} using two basic one-to-all searches on the reversed graph.
BOA* can then establish lower bounds on the cost of complete paths or \textbf{f}-values using the admissible heuristic \textbf{h}.
Although either of the two objectives can potentially play the key role in the bi-objective setting, standard BOA* usually chooses the first objective in the ($f_1,f_2$) order.
The search then expands all the promising nodes based on their cost estimates so as to ensure the node with the lexicographically smallest \textbf{f}-value is explored first. The algorithm terminates when there is no node in \textit{Open} while keeping all the non-dominated nodes associated with the $\mathit{goal}$ state in the solution set \textit{Sol}.
The main steps of the standard BOA* algorithm can be found in Algorithm~\ref{alg:boa_enh}, scripted with normal line numbers (without asterisk *) in black.
\begin{theorem}
BOA* computes a set of cost-unique Pareto-optimal solution paths \cite{ulloa2020simple}.
\end{theorem}
BOA* utilises an efficient strategy to check nodes for their dominance, originally employed in \cite{DBLP:journals/cor/PulidoMP15} for multi-objective search.
The idea is simple yet powerful.
Let us assume A* explores the graph in the ($f_1,f_2$) order, that is, nodes are extracted based on their $f_1$-value in order (with tie-breaking on $f_2$-values).
Meanwhile, $\mathit{x}$ and $\mathit{y}$ are two nodes associated with the same state or $s(x)=s(y)$ in the \textit{Open} list where $\mathit{x}$ is going to be expanded first, i.e., we have $f_1(x) \leq f_1(y)$.
Since both nodes have used the same heuristic value as $h_1(s(x))=h_1(s(y))$ to determine their cost estimate $f_1$, we can conclude $g_1(x) \leq g_1(y)$.
Therefore,
the second node will be dominated by the first node if $g_2(x) \leq g_2(y)$ as shown in \cite{DBLP:journals/cor/PulidoMP15} in detail.
BOA* takes advantage of this dimension reduction technique by systematically keeping track of the $g_2$-value of the last non-dominated node using $g_2^{min}(s(x))$ via line \ref{alg:boa_enh:min_r} of Algorithm~\ref{alg:boa_enh}.

BOA* can also prune some of the dominated nodes during the expansion with a similar reasoning via line \ref{alg:boa_enh:prune1} of Algorithm~\ref{alg:boa_enh}.
This is done by comparing the newly generated node of a state and the last expanded node of the state against their secondary costs $g_2$.
Furthermore, BOA* prunes unpromising nodes based on their cost estimate to the $\mathit{goal}$ state, which is known as pruning by bound.
Given $g_2^{min}(goal)$ as the upper bound of the secondary cost, partial paths will be pruned if the cost estimate of their complete paths to $\mathit{goal}$ on $g_2$ is greater than that of the last solution already stored in $g_2^{min}(goal)$.
Interested readers are referred to the standard BOA* algorithm in \cite{ulloa2020simple} for the detailed proof discussion.\par
\textbf{Challenges:}
Lazy dominance checking in BOA* slows down the operations in the \textit{Open} list and consumes more space. 
In contrast to the costly linear dominance checking approach where new nodes are checked against all of the previously generated nodes associated with a state before their insertion into the \textit{Open} list, BOA* may add a node for which we have an unexpanded dominant node in \textit{Open}.
Thus, the search generates more nodes (using extra memory), and the \textit{Open} list will inevitably be longer.
Moreover, BOA* is only able to search the graph in one direction and with a specific objective ordering, 
% normally in the forward direction with the ($f_1,f_2$) order, 
whereas there can be cases with better performance on the reverse objective ordering as shown in \cite{ulloa2020simple}.
Our preliminary experiments also reveal that searching backwards (from $\mathit{goal}$ to $\mathit{start}$) may lead to significant improvements in the overall runtime.
There are also some inefficiencies in BOA* which can be addressed with extra considerations.
As an example,
for the simple graph in Figure~\ref{fig:example}, BOA* needs to expand all intermediate states for each individual solution, despite the fact that some of them are not offering any alternative (non-dominated) path to $\mathit{goal}$ (eg. $s_2$).
% and might be common in the ending segment of multiple solution paths (eg. state $s_2$).
%
\section{Bi-directional Bi-objective A* Search}
Recent improvements in bi-directional heuristic search have introduced new techniques to reduce the number of necessary node expansions, such as Near-Optimal Bi-directional Search in \cite{DBLP:conf/aaai/SturtevantF18} and Dynamic Vertex Cover Bi-directional Search in \cite{DBLP:conf/aaai/ShperbergFSSH19}.
Given the single-objective nature of the conventional shortest path problem, none of the existing front-to-end or front-to-front algorithms can practically tackle the bi-objective shortest path problem without incorporating necessary modifications.
Moreover, those algorithms are not necessarily efficient for the bi-objective search as obtaining the solutions' cost would no longer be possible in $O(1)$.
In the conventional bi-objective setting where both searches work on the same objective, every state offers a set of non-dominated nodes (partial paths) in each direction, and handling frontier collisions (obtaining all of the complete $\mathit{start}$-$\mathit{start}$ joined paths of the state) would be an exhaustive process which can outweigh the speed-up achieved by expanding fewer nodes.
Our preliminary experiments also confirmed that the conventional front-to-end bi-directional search with an efficient partial paths coupling approach can potentially generate fewer nodes but shows poor performance compared to the unidirectional search scheme BOA*.\par
We now present our contributions to the problem by explaining our Bi-Objective Bi-directional A* search (BOBA*).
BOBA* employs two complementary (enhanced) uni-directional BOA* to search the solution space in both (forward and backward) directions with different objective orders (($f_1,f_2$) and ($f_2,f_1$)).
Therefore, since the algorithm does not perform partial paths coupling, we do not need to handle frontier collisions. In other words, each uni-directional BOA* is allowed to explore the entire graph towards the opposite end for each individual solution.
The high level structure of BOBA* is given in Algorithm~\ref{alg:boba_high}.
\begin{algorithm}[t]
\footnotesize
\caption{Bi-Objective Bidirectional A* (BOBA*) High-level}
\label{alg:boba_high}
\DontPrintSemicolon
% \SetAlgoLined
  \DontPrintSemicolon 
    \SetKwBlock{DoParallel}{do in parallel}{end}
    \KwIn{A problem instance (G, {\bf cost}, $s_{start}$, $s_{goal}$)}
    \KwOut{A set of cost-unique Pareto-optimal solutions}
    \DoParallel{
      $ h_1', ub_2' \leftarrow$ \textbf{cost}-bounded A* from $s_{start}$ to $s_{goal}$ on G in ($f_1,f_2$) order \label{alg:boba_high:init_1} \;
      $ h_2, ub_1 \leftarrow$ \textbf{cost}-bounded A* from $s_{goal}$ to $s_{start}$ on Reversed(G) in ($f_2,f_1$) order \label{alg:boba_high:init_2} \;
    }
    \DoParallel{
      $h_2', ub_1' \leftarrow$ \textbf{cost}-bounded A* from $s_{start}$ to $s_{goal}$ on G in ($f_2,f_1$) order \label{alg:boba_high:init_3} \;
      $h_1, ub_2 \leftarrow$ \textbf{cost}-bounded A* from $s_{goal}$ to $s_{start}$ on Reversed(G) in ($f_1,f_2$) order \label{alg:boba_high:init_4} \;
    }
    \DoParallel{
        $Sol \leftarrow$ BOA*\textsubscript{enh} for (G, \textbf{cost}, $s_{start}$, $s_{goal}$) with heuristics (${\bf h}, {\bf ub}, {\bf h'}$) in ($f_1,f_2$) order \;
        $Sol' \leftarrow$ BOA*\textsubscript{enh} for (Rev(G), \textbf{cost}, $s_{goal}$, $s_{start}$) with heuristics (${\bf h'}, {\bf ub'}, {\bf h}$) in ($f_2,f_1$) order \;
    }
\Return{$Sol+Sol'$}

\end{algorithm}
BOBA* first obtains the preliminary heuristics and then performs two individual searches that explore the graph in both directions concurrently. 
The output will then be the aggregation of solutions found in each search routine.
To avoid searching for the same cost-optimum paths in both directions, BOBA* always chooses different orders for each direction.
Figure~\ref{fig:pareto_front_bounded}(Left) depicts the way Pareto-optimal solutions are found based on two searches in the two orders.
Initial solutions ($sol_{init}$) at both ends are typically the minimum cost paths already obtained via the heuristic searches for each objective.
These cost-optimum paths can also initialise the global upper bounds ($ub_1, ub_2$) needed by the pruning by bound strategies in BOA*.
The upper bounds are updated (always decreasing) during the search every time a valid solution is found, and $sol_{last}$ is the last solution for which we have had $f_1 < ub_1$ and $f_2 < ub_2$.
\begin{definition}
For every state $s\in S$, ${\bf ub}(s)$ is the upper bound on ${\bf cost}$ of complementary paths from state $\mathit{s}$ to $\mathit{goal}$, eg., $ub_1(s)$ denotes the upper bound on $cost_1$.
\end{definition}
\begin{definition}
A path/node/state $\mathit{x}$ is invalid if its estimated costs ${\bf f}(x)$ are not in the search global upper bounds ($ub_1, ub_2$), i.e., $\mathit{x}$ is invalid if $f_1(x) \ge ub_1 $ or $f_2(x) \ge ub_2 $.
\end{definition}
\begin{figure}[t]
\begin{subfigure}{0.44\textwidth}
\includegraphics[]{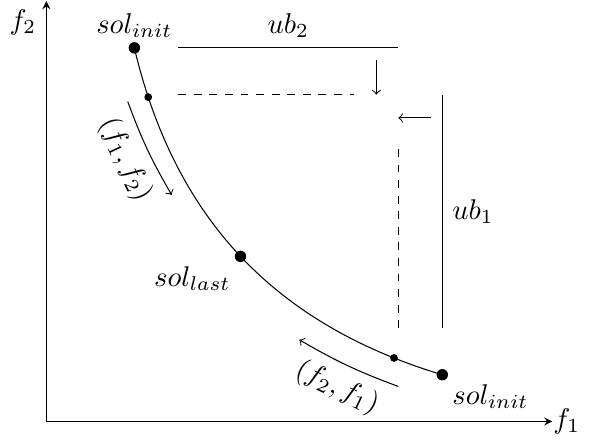}
\label{fig:pareto_front}
\end{subfigure}
\begin{subfigure}{0.54\textwidth}
\includegraphics[]{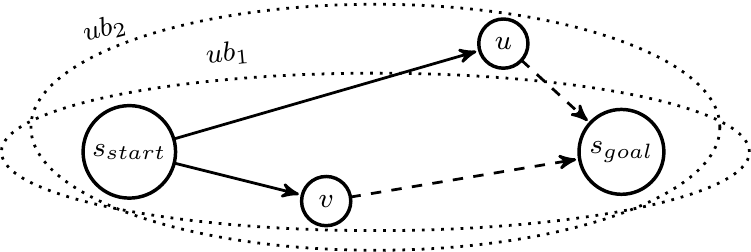}
\label{fig:bounded_search}
\end{subfigure}
\caption{ Left: Objective orders, bounds and Pareto-optimal solutions. Right: Schematic of states outside or inside of upper bounds. State \textit{u} is out of bounds for $f_1$ and will be discarded in $f_2$ search. }
\label{fig:pareto_front_bounded}
\end{figure}
\subsection{Preliminary Heuristics}
BOBA* requires both lower and upper bounds on the costs of complementary paths for each direction via four individual searches.
In each search, we calculate a state's upper bound to be the cost of the optimum path using the non-primary objective. 
For example, the optimum path to state $s$ for the first objective sets both $h_1(s)$ and $ub_2(s)$ (here $ub_2(s)$ is the cost of the path using the second objective).
BOA* traditionally uses two runs of Dijkstra's algorithm to initialise lower bounds.
For difficult cases, this initialisation time is usually outweighed by the main search time, but there can be simple cases where the total time of these heuristic searches dominates the main search time, especially in large instances.
As a more efficient initialisation approach, we replace Dijkstra's algorithm with cost-bounded A* (or cost-bounded Dijkstra without heuristics), as formally stated in Lemma~\ref{lemma:1} and shown in lines \ref{alg:boba_high:init_1}-\ref{alg:boba_high:init_4} of Algorithm~\ref{alg:boba_high}.
\begin{lemma} \label{lemma:1}
The preliminary A* search on $f_1$ (or $f_2$) can terminate before expanding a state with $f_1 > ub_1$ (or $f_2 > ub_2$).
\end{lemma}
\begin{proof} 
Assume that a forward BOA* is intended and, therefore, the corresponding heuristics (via two backward searches) are required.
If we start with two simple backward A* searches (one for each objective), each optimum $\mathit{start}$-$\mathit{goal}$ path gives us two bounds as $(h_1(s_{start}),ub_2(s_{start}))$ and $(h_2(s_{start}),ub_1(s_{start}))$.
Now, given $h_1(s_{start})$ and $h_2(s_{start})$ as the global lower bounds on $f_1$ and $f_2$-values respectively, we will have $f_1 \geq h_1(s_{start})$ and $f_2 \geq h_2(s_{start})$ for every $\mathit{start}$-$\mathit{goal}$ path.
Therefore, any state with a cost estimate of $f_1 > ub_1(s_{start})$ in the A* search on the first objective, and similarly $f_2 > ub_2(s_{start})$ in the search on the second objective, will be dominated by one of the optimum solutions.
On the other hand, since A* expands states in an increasing order of $\mathit{f}$-values, each heuristic search can terminate early with the first out-of-bound state, guaranteeing that all paths via unexplored states are already dominated.
% The same reasoning is correct for the reverse order and also the opposite direction.
\end{proof}
Algorithm~\ref{alg:boba_high} shows the parallel computation of all necessary heuristics in BOBA* in two phases.
In the first phase (lines \ref{alg:boba_high:init_1}-\ref{alg:boba_high:init_2}), we can execute our cost-bounded A* using any admissible heuristic for the primary objective ($f_1$ or $f_2$) and with tie-breaking on the secondary objective ($f_2$ or $f_1$).
Note that the upper bounds are unknown prior to the searches in phase one, i.e., we initially have $ub_1=ub_2=\infty$, but we can update our global upper bounds as soon as we establish the optimal solution in each search.
The initialisation step of BOBA* can be further improved for the heuristic searches in the opposite direction in phase two (lines \ref{alg:boba_high:init_3}-\ref{alg:boba_high:init_4}).
Once the necessary heuristics in one direction have been obtained, the heuristic search in the opposite direction can use the lower bounds obtained from the first round as more informed heuristics. That is, the second phase of our cost-bounded A* searches are normally executed faster.
Moreover, the opposite search in the second round can take advantage of the reduced search space resulting from the first round, delivering better quality heuristics without needing to expand already invalidated (out-of-bound) states.
Lemma~\ref{lemma:2} states this technique more formally.\par
\textbf{Example:}
State $\mathit{v}$ in Figure~\ref{fig:pareto_front_bounded}~(right) is within the upper bound of both objectives and will be expanded in the opposite direction.
However, state $\mathit{u}$ is observed out-of-bound for the first objective (but within the bound of the second objective) and will then be discarded if it is going to be expanded in the second round of our heuristic searches.
Note that violating at least one objective's upper bound is enough to mark nodes (or states) invalid.
\begin{lemma}\label{lemma:2}
In the preliminary A* search on $f_1$ (or $f_2$), states with an estimated cost of $f_2 > ub_2$ (or $f_1 > ub_1$) are not part of any solution path.
\end{lemma}
\begin{proof}
States with $f_2 > ub_2$ are dominated by the optimum path obtained for the first objective.
This means that unexplored states with an estimated cost of $f_2 > ub_2$ are all invalid.
% in the heuristic search on $f_1$ can safely be pruned in the heuristic search on $f_2$.
Therefore, the following search on $f_1$ can ignore expanding such states knowing that no non-dominated solution can be found via invalid states.
The same reasoning is valid for the reverse order.
\end{proof}
\subsection{Bi-directional Search}
BOBA* performs two enhanced BOA* concurrently, one from each direction. Algorithm~\ref{alg:boa_enh} shows the details of our first enhanced BOA* algorithm used in BOBA* (forward search in the ($f_1,f_2$) order).
Lines scripted in black are from the standard BOA* and the red lines with an asterisk (*) next to line numbers are our proposed enhancements.
To be consistent with the BOA* notation, we obtain the latest global upper bounds from $g_{min}$ values, i.e., we have $g_2^{min}(s_{goal})=ub_2$ and $g_1^{min}(s_{start})=ub_1$.
This is because the forward search on ($f_1,f_2$) updates $g_2^{min}(s_{goal})$ for every solution, whereas the backward search on ($f_2,f_1$) simultaneously updates $g_1^{min}(s_{start})$.
We also add a pruning criterion to discard nodes violating the primary upper bound $g_1^{min}(s_{start})$ in line \ref{alg:boa_enh:prune2}.
To achieve the backward search, we simply reverse the search direction and the objective ordering (see Appendix~\ref{app:boba_back}).
For example, instead of $g_2^{min}(s_{goal})$ and $h_1(s(x))$ in Algorithm~\ref{alg:boa_enh} we will have $g_1^{min}(s_{start})$ and $h_2'(s(x))$ respectively (the backward search establishes its $\mathit{f}$-values using ${\bf h'}$).
Note that each search has an independent \textit{Open} list.
Now we describe our contributions to the individual searches of BOBA* followed by their formal presentation in Lemmas~\ref{lemma:3}-\ref{lemma:4}.\par
\begin{algorithm}[t]
\footnotesize
\caption{Enhanced forward Bi-Objective A* (BOA*\textsubscript{enh}) in ($f_1,f_2$) objective ordering}
\label{alg:boa_enh}
\DontPrintSemicolon
\SetAlgoLined
 \KwInput{A problem instance (G, {\bf cost}, $s_{start}$, $s_{goal}$) and heuristics ({\bf h}, {\bf ub}, ${\bf h'}$)}
 \KwOutput{A set of cost-unique Pareto-optimal solutions}
 $Open \leftarrow \emptyset, Sol \leftarrow \emptyset$\ \;
 $g_1^{min}(s) \leftarrow g_2^{min}(s) \leftarrow \infty$ \ for each $s \in S$\ \;
 $x \leftarrow $ new node with $s(x) = s_{start}$\ \;
 $ {\bf g}(x) \leftarrow (0,0) $, $ {\bf f}(x) \leftarrow (h_1(s_{start}),h_2(s_{start})) $, $parent(x) \leftarrow Null$ \;
Add $x$ to $Open$\;
\While{$Open \neq \emptyset$}
{
 Remove a node $x$ with the lexicographically smallest ($f_1,f_2$) values from $Open$ \;
     \textcolor{red}{ 
       \label{alg:boa_enh:termination}\nlast
        \lIf{$f_1(x) \geq g_1^{min}(s_{start})$} 
      {  \textbf{break} }
      } 
     
      \lIf{$g_2(x) \geq g_2^{min}(s(x)) $ \textbf{or} $f_2(x) \geq g_2^{min}(s_{goal}) $ \label{alg:boa_enh:prune0}}  
     {\textbf{continue}} 
     
     \textcolor{red}{
     \label{alg:boa_enh:tuning}\nlast
     \lIf{$g_2^{min}(s(x)) = \infty$}
    { $h_1'(s(x)) \leftarrow g_1(x)$ } 
   }
    $g_2^{min}(s(x)) \leftarrow g_2(x)$ \label{alg:boa_enh:min_r} \;
    \If{$s(x) = s_{goal}$} 
     { 
        \textcolor{red}{
        \label{alg:boa_enh:tie0}\nlast   $z \leftarrow $ last node in $Sol$ \; 
        \nlast   \lIf{$(z \neq Null $ \textbf{and} $f_1(z) = f_1(x))$} { Remove $z$ from $Sol$ } \label{alg:boa_enh:tie1}
        }
        Add $x$ to $Sol$ \;
        \textbf{continue}
    }
     \textcolor{red}{
     \label{alg:boa_enh:early_sol}\nlast
     \If{$g_2(x) + ub_2(s(x)) < g_2^{min}(s_{goal})$}
    {     \nlast  $g_2^{min}(s_{goal}) \leftarrow g_2(x) + ub_2(s(x))$ \;
         \nlast   $z \leftarrow $ last node in $Sol$ \; \label{alg:boa_enh:tie2}
         \nlast   \lIf{$(z \neq Null$ \textbf{and} $f_1(z) = f_1(x))$} { Remove $z$ from $Sol$ } \label{alg:boa_enh:tie3}
         \nlast   Add $x$ to $Sol$ \;
         \nlast\lIf{$h_1(s(x))= ub_1(s(x))$}
       { \textbf{continue} } \label{alg:boa_enh:unique_path}
    }
    }
    
         \For{$all\ t \in Succ(s(x))$}
        { $y \leftarrow $ new node with $s(y) = t$ \;
             ${\bf g}(y) \leftarrow {\bf g}(x) + {\bf cost} (s(x),t)$ \;
             ${\bf f}(y) \leftarrow {\bf g}(y) + {\bf h} (t)$ \;
             $parent(y) \leftarrow x$ \;
            \lIf{$g_2(y) \geq g_2^{min}(t) $ \textbf{or} $f_2(y) \geq g_2^{min}(s_{goal}) $ \label{alg:boa_enh:prune1}} 
            { \textbf{continue} }
            
            \textcolor{red}{\label{alg:boa_enh:prune2}\nlast
            \lIf{$f_1 \geq g_1^{min}(s_{start}) $ }
            { \textbf{continue} } 
            }
             Add $y$ to $Open$\;
        }
}
\Return{$Sol$}\; 
\end{algorithm}
\textbf{Early solution update:}
{This strategy allows the search to update the secondary upper bound and possibly establish a solution before reaching the $\mathit{goal}$ state.}
This is done via line \ref{alg:boa_enh:early_sol} of Algorithm~\ref{alg:boa_enh} by coupling nodes with their complementary shortest path to $\mathit{goal}$.
If the joined path is valid, the corresponding node is then temporarily added to the solution set knowing that solution nodes with a state other than $\mathit{s_{goal}}$ (or $s(x) \neq s_{goal}$) must be joined with their complementary shortest path.
This strategy can be further improved by not expanding nodes for which we have a unique non-dominated complementary path.
This heuristic is incorporated in line \ref{alg:boa_enh:unique_path} and is formalised in Lemma~\ref{lemma:3}.\par
\textbf{Secondary heuristic tuning:}
Bi-directional search provides our algorithm with a great opportunity to further improve the quality of the preliminary heuristics.
Since the main search of BOBA* has more information about non-dominated paths to states and constantly updates upper bounds, there can be more outliers that our preliminary heuristic searches are not aware of.
% Our main search function can benefit from the main property of BOA*, which is its ability to find non-dominated valid paths associated with nodes in order.
Therefore, benefiting from the main property of BOA* (finding non-dominated nodes in order), we can tune our findings over the preliminary searches and empower the pruning by bound strategy of the concurrent search in the opposite direction.
This tuning is done in O(1) time by updating the secondary heuristics of the reverse direction via line \ref{alg:boa_enh:tuning} of Algorithm~\ref{alg:boa_enh}.
Note that $h_1'$ denotes the secondary heuristic in the backward search where BOBA* uses $f_2$ as its primary cost.
We discuss the correctness of this technique in Lemma~\ref{lemma:4}.
%
% \noindent
\begin{figure}[t]
\begin{subfigure}{0.49\textwidth}
\includegraphics[]{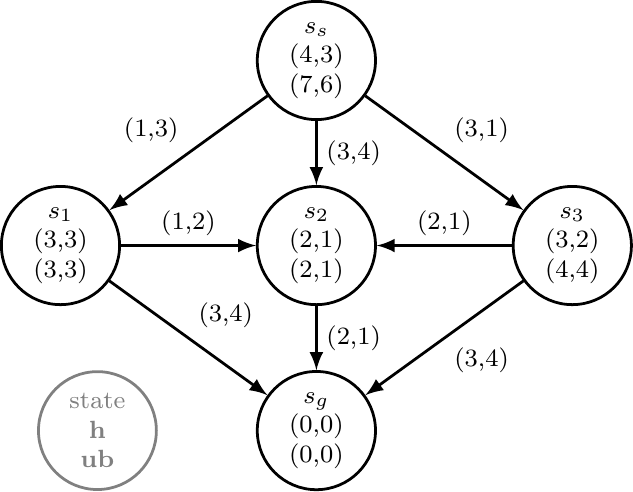}
\end{subfigure}
\begin{subfigure}{0.49\textwidth}
\centering
\footnotesize
\renewcommand{\arraystretch}{1}
\begin{tabular}{ c   l  c   c}
    \toprule
     & \textit{Open} list & \textit{Sol}  & Update \\
    It. &  ($s(x)$, \textbf{g}(x), \textbf{f}(x)) & found & $g^{min}_2(s_{g})$\\
    \midrule
    1 & $\uparrow$($s_{s}$, (0,0), (4,3)) & (4,6)  & $\infty \rightarrow$ 6\\
    \midrule
    2 & $\uparrow$($s_2$, (3,4), (5,5)) & (5,5) & 6 $\rightarrow$ 5\\
     & \ ($s_3$, (3,1), (6,3))  &  & \\
    \midrule
    3 & $\uparrow$($s_3$, (3,1), (6,3))  &  & \\
    %  & \ ($s_{g}$, (8,6), (8,6))  &   &  \\
    \midrule
    4 & $\uparrow$($s_2$, (5,2), (7,3))  &  (7,3) & 5 $\rightarrow$ 3 \\
    \midrule
    5 & empty  &   &  \\
    \bottomrule
\end{tabular}
\newline
\vspace*{0.3 cm}
\newline
\setlength{\tabcolsep}{4.5pt}
\begin{tabular}{|l| *{5}{c|} }
    \toprule
    parent arrays & $s_s$ & $s_1$ & $s_2$ & $s_3$ & $s_g$\\
    \midrule
    \texttt{par\_state} & [Null] & & [$s_s$, $s_3$] & [$s_s$] & \\
    \midrule
    \texttt{par\_path\_id}& [Null] & & [1, 1] & [1] & \\
    \bottomrule
\end{tabular}

\end{subfigure}
\caption{ Left: An example graph with \textbf{cost} on the edges, and with (state, \textbf{h}, \textbf{ub}) inside the nodes.
Right: Status of the \textit{Open} list, new solution (\textit{Sol}) and secondary upper bound $g_2^{min}(s_g)$ in every iteration (It.) for the forward search on the $(f_1,f_2)$ ordering. Symbol $\uparrow$ beside nodes denotes the expanded min-cost node.
The second table shows the status of the parent arrays of the states when the search terminates.}
\label{fig:example}
\end{figure}
\par
\textbf{Example:}
We explain these strategies by just running the forward search of BOBA* for the graph in Figure~\ref{fig:example} and iterations in Table~\ref{fig:example}.
In the first iteration, the forward search explores the node associated with the start state $\mathit{s_s}$. 
Since the primary (heuristic) cost-optimum path from $\mathit{s_s}$ is initially valid, the search immediately updates its secondary upper bound via the early solution update strategy by setting $g_2^{min}(s_g) = 6$ and adds the node into the \textit{Sol} set with costs ($4,6$).
During the $\mathit{s_s}$ expansion, we notice that the extended path for state $\mathit{s_1}$ is invalid ($f_2(y) \ge g_2^{min}(s_g)$ or $3+3 \ge 6$).
Therefore, the partial path to $\mathit{s_1}$ is pruned meaning that $\mathit{s_2}$ is no longer reachable via its primary cost optimum path.
Nodes generated for states $\mathit{s_2}$ and $\mathit{s_3}$, however, are successfully added to \textit{Open}.
In the second iteration, the algorithm picks the node associated with $\mathit{s_2}$ (with higher priority).
Now, since this is the first time we see $\mathit{s_2}$ being expanded, and since future visits will always have higher costs (via $\mathit{s_3}$ with $g_1=5$ for example), we can update the lower bound of reaching $\mathit{s_2}$ from $\mathit{s_s}$ knowing that all possible shorter paths have already been invalidated.
This is done by updating $h_1'(s_2)=3$.
Note that from the preliminary heuristics, we already had $h_1'(s_2)=2$ (lower bound from $\mathit{s_s}$ to $\mathit{s_2}$).
After this update, the backward search would have better quality secondary lower bounds and can effectively prune more nodes (the backward primary heuristic is $h_2'$).
We skip the backward search for now and continue with our forward expansions.
As coupling the node (associated with $s_2$) with its (complementary) primary cost-optimum path yields a valid complete path, the search updates its secondary upper bound and temporarily adds the node to the \textit{Sol} set with costs ($5,5$). 
The search also skips expanding $s_2$ as it does not offer any non-dominated path to $s_g$.
In the third iteration, the node associated with state $s_3$ is picked.
This time, $s_3$ is expanded since coupling does not yield valid path.
During the $s_3$ expansion, the search finds $s_g$ invalid but adds $s_2$ into \textit{Open}.
In the fourth iteration, the node associated with $s_2$ is the only node in \textit{Open} which reveals the final solution with costs ($7,3$), again with the early solution update strategy.
This last solution also verifies that the temporary solution found in the second iteration is now a valid non-dominated solution, since the primary cost of the last solution is larger than that of the second solution ($5<7$).
\par
Now we formally prove the correctness of the presented techniques as follows.
\begin{lemma}\label{lemma:3}
At every iteration, if $g_2(x) + ub_2(s(x)) < g_2^{min}(s_{goal})$, the next solution has a primary cost of $f_1(x)$ and a secondary cost of at most $g_2(x) + ub_2(s(x))$. Node $\mathit{x}$ is also a terminal node if $h_1(s(x))= ub_1(s(x))$.
\end{lemma}
\begin{proof}
If the joined path is valid (its secondary cost is within the bounds),
expanding nodes on the complementary shortest path will definitely navigate us to $s_{goal}$ with a valid secondary cost as they offer the same $f_1$-value.
This means we can efficiently update the secondary upper bound earlier assuming that a potential solution path is already established.
Therefore, valid joined paths determine the primary cost $f_1$ of the next solution along with setting a new upper bound for the secondary cost $f_2$.
Furthermore, given the secondary cost as a tie-breaker in the preliminary heuristic searches, states with $h_1(s(x))= ub_1(s(x))$ would only offer one complementary path optimum for both objectives.
As none of the states on the complementary path would offer an alternative non-dominated path to $\mathit{s_{goal}}$, the search can save time by not expanding such terminal nodes.
Therefore, nodes with $h_1(s(x))= ub_1(s(x))$ are terminal nodes if they appear on any solution path.
\end{proof}
The early solution update strategy above guarantees that the primary cost $f_1$ of the next solution is determined by the joined path, but this does not apply to its secondary cost. E.g., we might see two consecutive temporary solutions with the same $f_1$-value but different (sequentially) valid secondary costs.
Therefore, the search needs to make sure that the previously added solution is not dominated by the next potential solution, and if it is dominated, it must be removed from the non-dominated solution set $Sol$.
We address this matter in O(1) time by checking our last (temporary) solution against new solutions for dominance as shown in lines \ref{alg:boa_enh:tie0}-\ref{alg:boa_enh:tie1} and \ref{alg:boa_enh:tie2}-\ref{alg:boa_enh:tie3} of Algorithm~\ref{alg:boa_enh}.
This pruning is formally stated in the following Lemma~\ref{lemma:3_1}.
\begin{lemma}\label{lemma:3_1}
Given $z$ and $x$ as the last and new temporary solution nodes respectively, node $z$ represents a non-dominated solution if $f_1(z) < f_1(x)$.
The temporary solution node $z$ is dominated by $x$ if $f_1(z) = f_1(x)$.
\end{lemma}
\begin{proof}
Since the secondary cost of the new solution $x$ is already checked to be smaller than that of the last solution $z$ stored in $g_2^{min}(s_{goal})$, we have $f_2(x) < g_2^{min}(s_{goal})$ if $s(x) = s_{goal}$ or $g_2(x) + ub_2(s(x)) < g_2^{min}(s_{goal})$ if $s(x) \neq s_{goal}$.
On the other hand, since $x$ is the new potential solution and the search explores nodes in an increasing order of $f_1$-values, we must have $f_1(z) \leq f_1(x)$.
Therefore, if $f_1(z) < f_1(x)$, we can see that the temporary solution $z$ is now a non-dominated solution.
Otherwise, if $f_1(z) = f_1(x)$, the last solution $z$ is dominated by the new solution $x$ because the new solution offers a lower secondary cost.
\end{proof}
We now show the correctness of the heuristic tuning approach in BOBA*.
\begin{lemma}\label{lemma:4}
The secondary heuristic tuning maintains the correctness of A* heuristics.
\end{lemma}
\begin{proof}
BOBA* expands partial paths in the increasing order of $\mathit{f}$-values.
This means the first expanded node of every state is guaranteed to have the minimum valid primary cost $g_1$ in each search direction, and all of the following valid nodes will have a larger primary cost.
Moreover, since BOBA* uses different objective ordering for its searches, updated lower bounds in one direction represent the secondary heuristics of the other direction.
Therefore, we can guarantee that the updated secondary heuristic is still admissible as there will not be any min-cost path to states better than what their first expanded node presents.
Furthermore, the tuning strategy only updates the secondary heuristics of the opposite search, i.e., $h_1'(s(x))$ in the forward and $h_2(s(x))$ in the backward search.
Therefore, the preliminary primary heuristics $h_1(s(x))$ and $h_2'(s(x))$ are unchanged and the A* searches are correct.
\end{proof}
Considering the correctness of the enhancements presented above, we now show the correctness of our BOBA* algorithm.
\begin{theorem}
BOBA* returns a set of cost-unique non-dominated solution paths.
\end{theorem}
\begin{proof}
BOBA* executes two enhanced BOA* searches concurrently, each capable of finding all of the solutions.
Therefore, we just need to show the correctness of the stopping criteria.
Each (enhanced) BOA* searches the primary objective's domain in the increasing order of $\mathit{f}$-values and continually shrinks the secondary objective's domain every time a valid solution is found.
Furthermore, since BOBA* shares the upper bounds between its searches, each search can terminate with the first node violating the main objective's upper bound (and consequently other unexplored nodes with larger $\mathit{f}$-values in \textit{Open}) knowing that the rest of the objective's domain has already been investigated by the concurrent search (see Figure~\ref{fig:pareto_front_bounded}).
Therefore, the aggregation of the solutions found in each search yields a complete set of cost-unique non-dominated solutions.
\end{proof}
\section{Practical Considerations}
As BOA* enumerates all non-dominated paths, the size of \textit{Open} can grow exponentially over the course of search. 
Furthermore, the huge number of nodes in difficult cases may result in major memory issues.
For instance, for one particular case in our experiments BOA* generates two billion nodes.
We now present two techniques to handle search nodes more efficiently.\par
\textbf{More efficient \textit{Open} list:}
To achieve faster operations in our \textit{Open} lists, 
since the lower and upper bounds on the $\mathit{f}$-values of the nodes in BOBA* are known prior to its main searches,
we replace the conventional heap-based lists with fixed-size bucket lists without tie-breaking \cite{DBLP:journals/ior/DenardoF79}.
In contrast to other problems where the bucket list is regularly resized and the list is sparsely populated, for the significant number of (cost-bounded) nodes in our problem we expect to see almost all of the buckets filled.
Note that the search may also expand dominated nodes if they are not extracted in a lexicographical order (i.e., nodes are sorted based on their primary cost only), but BOBA* can still obtain cost-unique solutions via the dominance checks incorporated in lines \ref{alg:boa_enh:tie0}-\ref{alg:boa_enh:tie1} and \ref{alg:boa_enh:tie2}-\ref{alg:boa_enh:tie3} of Algorithm~\ref{alg:boa_enh} as formally stated in the following Lemma~\ref{cor:1}.
\begin{lemma} \label{cor:1}
BOBA* is able to obtain cost-unique solutions even without tie-breaking in its \textit{Open} lists.
\end{lemma}
\begin{proof}
Let $z$ and $x$ be two solution nodes where $f_1(z)=f_1(x)$ and $z$ is dominated by $x$.
Without any tie-breaking, the search may temporarily add dominated node $z$ to the solution set first.
In the next iterations, when $x$ is extracted, the search performs a quick dominance check by comparing the $f_1$-value of the new node $x$ against that of the previous solution $z$ and substitutes the dominated solution with the new solution $x$ if $f_1(z)=f_1(x)$, as already shown in the early solution update strategy and Lemma~\ref{lemma:3_1} in detail.
Therefore, BOBA* computes cost-unique non-dominated solutions even without tie-breaking.
\end{proof}
% An alternative approach to make the \textit{Open} list lighter could be only keeping the states' best node candidate in the \textit{Open} list.
%
\textbf{Memory efficient backtracking:}
Creating nodes is necessary to appropriately navigate the search to valid solution paths.
Each new node occupies a constant amount of memory and conventionally contains essential information about paths such as costs and also back-pointers for solution path construction.
% such as costs of the partial path, corresponding node and a pointer to the parent node needed for the solution path construction.
Considering the difficulty of the problem and the significant number of generated nodes, we suggest a more memory efficient approach for the solution path construction in BOBA*.
Since BOBA* only expands nodes once, we propose to recycle the memory used to store heavy processed nodes, while storing their backtracking information in other compact data structures. 
This technique results in a major reduction in memory use as part of the nodes' information would no longer be required for backtracking. 
% Even if needed, the information can easily be retrieved via backtracking.
%the number of generated nodes.
% as it allows the search to extract the minimal information required for path construction and recycle the processed node for later expansions.
% To this end, we use a compact data structure to keep track of expanded paths associated with states using two optimised dynamic arrays.
We explain our compact approach using an example from Figure~\ref{fig:example}. 
Assume that in the second iteration of the algorithm, we want to store the backtracking information of the node corresponding to $s_2$ with $s_s$ as the parent state.
To this end, we keep two (initially empty) dynamic arrays for each state: one to store the parent state of the node \texttt{par\_state}, and another to look up the corresponding path index in the parent state \texttt{par\_path\_id}.
For our example, since the first path to $s_2$ is derived from the first non-dominated path of $s_s$, we store this sequence in $s_2$ as \texttt{par\_state[1]=$s_s$} and \texttt{par\_path\_id[1]=1}.
Similarly, for the second expansion of $s_2$ with $s_3$ as the parent in the fourth iteration, we update $s_2$ arrays, this time with \texttt{par\_state[2]=$s_3$} and \texttt{par\_path\_id[2]=1}.
Figure~\ref{fig:example} also shows the situation of our parent arrays when the forward search terminates.
As a further optimisation, we can store the index of incoming edges (which are usually very small integers) instead of parent states.
% The final size of the compact structure varies from instance to instance but it is on average four times lighter than conventional nodes.
We will investigate the impacts of this compression on memory usage in the following section.
\section{Empirical Study and Analysis}
We compare our BOBA* with recent algorithms designed for the bi-objective search problem.
The selected algorithms are the bi-objective variants of Dijkstra's algorithm (Dij) and bi-directional Dijkstra (Bi-Dij) from \cite{DBLP:journals/eor/Sedeno-NodaC19}, and bi-objective A* (BOA*) from \cite{ulloa2020simple}.
We use all seven benchmark instances of \cite{DBLP:journals/eor/Sedeno-NodaC19} which include 700 random $\mathit{start}$-$\mathit{goal}$ pairs from the large road networks in the 9th DIMACS challenge \cite{dimacs9th}
with (\textit{distance}, \textit{time}) as objectives. 
To further challenge the algorithms, we used the competition's random pair generator to design an additional set of 300 test-cases for the larger networks in the DIMACS instances: E, W and CTR (100 cases each) with up to 15~M nodes and 33~M edges. 
The details of the instances can be found in
\cite{dimacs9th}.
\par
\textbf{Implementation:}
We implemented our BOBA* algorithms based on a parallel framework using two cores in C++ and used the C implementations of the Dij, Bi-Dij and BOA* algorithms kindly provided to us by their authors.
We also fixed the scalability and tie-breaking issues in the standard BOA* algorithm before running the experiments.
All code was compiled with O3 optimisation settings using the GCC7.4 compiler.
Our codes are publicly available\footnote{\url{https://bitbucket.org/s-ahmadi}}.
We ran the 1,000 experiments on an Intel Xeon E5-2660V3 processor running at 2.6~GHz and with 128~GB of RAM, under the SUSE Linux 12.4 environment and with a one-hour timeout.\par
\textbf{BOA* analysis:}
The search in BOA* can be performed in different directions and objective orders, resulting in four variants.
We also consider the virtual best version of the four, called BOA*\textsubscript{best} (essentially assuming an oracle that could select the best variant).
Figure~\ref{fig:BOA_order} is a cactus plot comparing the performance of all BOA* variants including the virtual best, showing how many instances can be solved in a given time (the plot only shows the longest running 300 instances).
Backward BOA* in the $(f_1,f_2)$ order (in green) is the weakest variant, but the other variants perform quite similarly, and it is difficult to declare a clear winner.
The performance of the virtual variant BOA*\textsubscript{best} shows that an ideally-tuned BOA* can be up to two times better than its standard version on average, but is still unable to solve 15 cases to optimality within the time limit (see Table~\ref{table:result_time}).
\begin{figure}[t]
\begin{floatrow}
\ffigbox{%
  \includegraphics[]{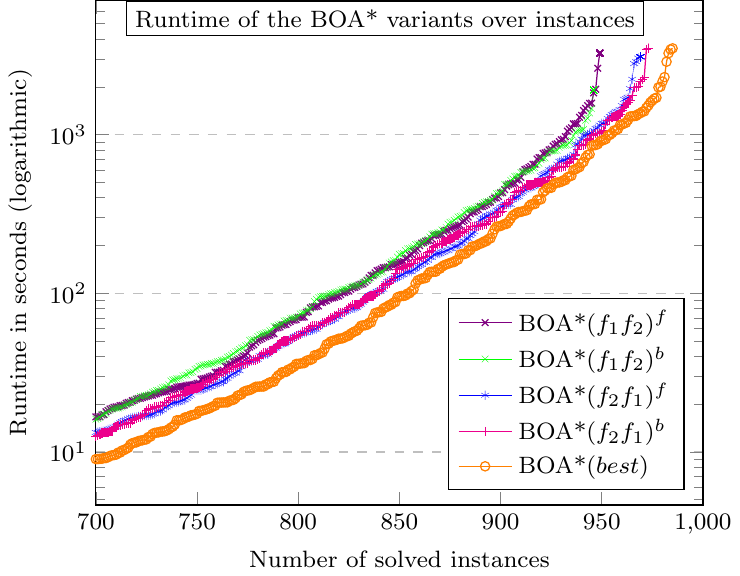}
}{%
%   \vspace{-1.5 \baselineskip}
  \caption{ Performance of the BOA* variants.}%
  \label{fig:BOA_order}
}
\capbtabbox{%
\footnotesize
\begin{tabular}{ l  l  r r }
\toprule
    
%     Saving & & \multicolumn{2}{c||}{ Mem. (MB)}
%   & &  \multicolumn{2}{c|}{ Mem. (MB)}
%     \\ \cline{3-4} \cline{6-7}
%   Approach & Inst. &  Avg. & Max
%   & Inst. & Avg. & Max
%     \\
     & Saving & \multicolumn{2}{c}{ Mem. (MB)}\\ \cmidrule{3-4}
   Inst. & Approach &  Avg. & Max\\
    \midrule
% FLA & Conv. & 130 & 1762 \\
%   & Compact &\textbf{25} & \textbf{314} \\
% \midrule

NE & Conv. &  307 & 7618\\
&  Compact & \textbf{61} & \textbf{1186} \\
\midrule
CAL & Conv.  &  326 & 5421\\
&  Compact  & \textbf{62} & \textbf{1009}\\
\midrule

LKS & Conv.  &  5585 & 54331\\
&  Compact & \textbf{955} & \textbf{9411}\\
\midrule

E & Conv.  & 5836 & 62963\\
&  Compact  & \textbf{999} & \textbf{10895}\\
\midrule

W & Conv.  & 5877 & 79602\\
&  Compact  & \textbf{925} & \textbf{10498}\\
\midrule

CTR & Conv. & 15835 & 99108\\
&  Compact & \textbf{2662} & \textbf{21749}\\
\bottomrule

\end{tabular}
}{%
  \caption{ BOBA* memory usage for the conventional and compact backtracking approaches.}%
  \label{table:result_memory}
}
\end{floatrow}
\end{figure}
%
%
% \begin{minipage}[t]{0.99\textwidth}
%   \begin{minipage}[b]{0.45\textwidth}
%     \centering
%     \includegraphics[]{Fig4.pdf}
%     \captionof{figure}{ Performance of the BOA* variants.}
%     \label{fig:BOA_order}
%   \end{minipage}
%   \hfill
%   \begin{minipage}[b]{0.45\textwidth}
%     \centering
%     \input{Result_memory_ESA}
%     \captionof{table}{ BOBA* memory usage for the conventional and compact backtracking approaches.}
%     \label{table:result_memory}
%     \end{minipage}
%   \end{minipage}
\par
\textbf{Memory:}
We investigate the impact of our compact approach for the solution paths construction in BOBA*.
Table~\ref{table:result_memory} compares the memory usage of our proposed compact approach against the conventional backtracking approach on part of the benchmark instances.
In order to measure the overall space requirement of the main search, we ignore the memory required for graph construction, shared libraries and heuristics, allocated prior to the search.
The results show that BOBA* can solve all of the instances with both approaches within the time limit, but the compact approach runs slightly faster and is five times more efficient on average in terms of memory.
For the most difficult case in the experiments, the required memory of the compact approach can be as low as 21~GB (allocating 15M nodes with recycling) where the conventional approach needs 96~GB (allocating 1B nodes).
Note that both approaches nearly expand the same number of nodes to solve the cases to optimality.
\par
\textbf{BOBA* performance:}
We compare the performance of our parallel BOBA* algorithm with the state-of-the-art Dij, Bi-Dij and BOA* algorithms from the literature.
Table~\ref{table:result_time} shows the summary of experimental results for the 100 cases of each instance.
For unsolved cases, we generously assume a runtime of one hour (the timeout).
We also report the average memory usage of the main search of each algorithm over solved cases, ignoring the space allocated for their initialisation phase.
The results in Table~\ref{table:result_time} show that the standard BOA* algorithm runs faster, needs less memory compared to both Dij and Bi-Dij algorithms and solves more instances.
However, our new BOBA* outperforms BOA* in all of the instances, showing an (arithmetic) average speed up of 16 over all of the individual cases.
For the average runtime of all instances, BOBA* is around five times faster than BOA*.
We also compare the algorithms' performance over the solved instances for both CPU time and memory usage in Figure~\ref{fig:performance}.
As shown for both metrics, BOBA* delivers superior performance to its competitors by solving all of the instances to optimality within the time limit and with a maximum memory usage of 21~GB, compared to the nearly full (128~GB) memory usage of other algorithms in difficult instances.
BOBA* also shows a massive speed up in the easy cases due to its efficient initialisation phase.
It can solve 282 cases before BOA* solves its easiest case.
Moreover, the figure shows that
% , for the specific number of instances,
BOBA* completes the task eight times more efficiently in terms of memory than BOA* on average.
Note that because of the difficulties in reporting the memory usage, we allow 1~MB tolerance in our experiments.
\begin{table*}[t]
\centering
\footnotesize
\setlength{\tabcolsep}{3.pt}
\renewcommand{\arraystretch}{1}
\begin{tabular}{l | l r *{3}{r} r | l  r *{3}{r}  r}
\toprule

    \multicolumn{3}{c}{}  & \multicolumn{3}{c}{ Runtime (s)} & \multicolumn{1}{c}{Mem.} &
    \multicolumn{2}{c}{}   & \multicolumn{3}{c}{ Runtime (s)} & Mem.
    \\ \cmidrule{4-6} \cmidrule{10-12}
    Alg. & Ins. & $|S|$ & Min & Avg. & Max & Avg. &
    Ins. & $|S|$ & Min & Avg. & Max & Avg.\\
    \midrule
BiDij & NY & 100 & 0.53 & 0.92 & 6.66 & 21 & CAL  & 98 & 4.04 & 168.58 & 3600.00 & 1206 \\ 
Dij & & 100 & 0.31 & 1.47 & 16.98 & 75 &  & 100 & 2.73 & 88.93 & 1105.57 & 4283 \\
BOA* & & 100 & 0.11 & 0.22 & 1.70 & 9 &  & 100 & 0.89 & 24.50 & 538.40 & 893 \\
BOA*\textsubscript{best} &  & 100 & 0.11 & 0.16 & 0.65 & 6 &  & 100 & 0.89 & 8.63 & 187.04 & 498 \\ 
BOA*\textsubscript{enh} & &100 & 0.01 & 0.12 & 0.67 & 2 &  & 100 & 0.02 & 6.69 & 147.44 & 67 \\
BOBA*\textsubscript{1c} & &100 & 0.01 & 0.11 & 0.59 & 7 &  & 100 & 0.01 & 6.08 & 116.10 & 97 \\
BOBA* & & 100 & \textbf{0.00} & \textbf{0.08} & \textbf{0.40} & 2  &  & 100 & \textbf{0.00} & \textbf{3.75} & \textbf{64.80} & 62  \\

% \midrule
\midrule

BiDij & BAY & 100 & 0.61 & 1.32 & 11.96 & 35 & LKS  & 69 & 6.12 & 1610.14 & 3600.00 & 2597 \\ 
Dij & & 100 & 0.36 & 2.01 & 19.71 & 107 &   & 81 & 4.13 & 936.23 & 3600.00 & 11394 \\
BOA* & & 100 & 0.13 & 0.38 & 4.10 & 19 &   & 89 & 1.30 & 528.12 & 3600.00 & 4854 \\
BOA*\textsubscript{best} & & 100 & 0.13 & 0.23 & 1.26 & 13 &  & 100 & 1.28 & 224.42 & 3500.80 & 9374 \\ 
BOA*\textsubscript{enh} & &100 & 0.01 & 0.19 & 1.20 & 3 &  & 100 & 0.02 & 97.05 & 1077.96 & 787 \\
BOBA*\textsubscript{1c} & &100 & 0.01 & 0.19 & 1.08 & 10 &  & 100 & 0.02 & 129.64 & 1488.41 & 1123 \\
BOBA* & & 100 & \textbf{0.00} & \textbf{0.13} & \textbf{0.86} & 2 &   & 100 & \textbf{0.00} & \textbf{69.68} & \textbf{812.17} & 955  \\
\midrule

BiDij & COL & 100 & 0.84 & 6.84 & 147.55 & 118 & E  & 64 & 8.08 & 1611.10 & 3600.00 & 2223 \\ 
Dij & & 100 & 0.52 & 6.27 & 111.81 & 348 &   & 82 & 5.48 & 1034.61 & 3600.00 & 16156 \\
BOA* & & 100 & 0.19 & 1.20 & 20.53 & 77 &   & 89 & 1.72 & 552.64 & 3600.00 & 5299 \\
BOA*\textsubscript{best} &  &100 & 0.18 & 0.58 & 7.03 & 49 &  & 98 & 1.72 & 293.27 & 3600.00 & 8701 \\ 
BOA*\textsubscript{enh} & &100 & 0.01 & 0.54 & 10.46 & 7 &  & 100 & 0.02 & 110.69 & 1684.94 & 850 \\
BOBA*\textsubscript{1c} & &100 & 0.02 & 0.42 & 5.50 & 21 &  & 100 & 0.02 & 143.08 & 1818.63 & 1160 \\
BOBA* & & 100 & \textbf{0.00} & \textbf{0.34} & \textbf{6.58} & 6 &   & 100 & \textbf{0.00} & \textbf{75.94} & \textbf{952.32} & 999  \\
\midrule

BiDij & FLA & 100 & 2.11 & 51.49 & 1088.49 & 808 & W &  69 & 14.38 & 1585.11 & 3600.00 & 3476 \\ 
Dij & & 100 & 1.37 & 52.34 & 1048.67 & 2630 &   & 74 & 10.04 & 1220.44 & 3600.00 & 12722 \\
BOA* & & 100 & 0.48 & 6.42 & 153.07 & 276 &   & 94 & 3.14 & 416.94 & 3600.00 & 7705 \\
BOA*\textsubscript{best} &  & 100 & 0.48 & 3.22 & 36.32 & 202 &  & 98 & 3.14 & 253.85 & 3600.00 & 8043 \\ 
BOA*\textsubscript{enh} & &100 & 0.01 & 2.05 & 34.47 & 22 &  & 100 & 0.04 & 93.16 & 1792.57 & 784 \\
BOBA*\textsubscript{1c} & & 100 & 0.01 & 2.06 & 27.98 & 43 &  & 100 & 0.04 & 130.81 & 1834.67 & 1134 \\
BOBA* & & 100 & \textbf{0.00} & \textbf{1.31} & \textbf{19.86} & 25  &   & 100 & \textbf{0.02} & \textbf{70.41} & \textbf{971.67} & 925  \\
\midrule

BiDij & NE & 99 & 3.31 & 181.67 & 3600.00 & 1367 & CTR  & 48 & 40.41 & 2666.66 & 3600.00 & 4904 \\ 
 Dij & & 100 & 2.18 & 68.41 & 1306.04 & 3281 &  & 51 & 29.29 & 2163.50 & 3600.00 & 16149 \\
 BOA* & & 100 & 0.73 & 16.83 & 332.36 & 587 &  & 77 & 8.46 & 1124.03 & 3600.00 & 9828 \\
BOA*\textsubscript{best} &  & 100 & 0.70 & 10.51 & 332.01 & 533 &  & 89 & 8.46 & 745.16 & 3600.00 & 12418 \\ 
 BOA*\textsubscript{enh} & &100 & 0.02 & 4.79 & 97.25 & 49 &  & 100 & 0.03 & 340.50 & 2953.12 & 2178 \\
 BOBA*\textsubscript{1c} & &100 & 0.02 & 5.71 & 154.51 & 82 &  & 98 & 0.03 & 461.07 & 3600.00 & 2644 \\
 BOBA* & & 100 & \textbf{0.00} & \textbf{3.41} & \textbf{90.01} & 61  &  & 100 & \textbf{0.02} & \textbf{246.01} & \textbf{2496.95} & 2662  \\
\bottomrule
\end{tabular}
\caption{\small 
Number of solved cases ($|S|$), runtime (in seconds) and average memory usage (Mem.) of algorithms over instances (Ins.).
Memory in MB for the main search over solved cases.
% We set 3600 seconds (1h) for the runtime of unsolved cases.
}
\label{table:result_time}
\end{table*}
\begin{figure}[t]
\begin{subfigure}{0.49\textwidth}
\includegraphics[]{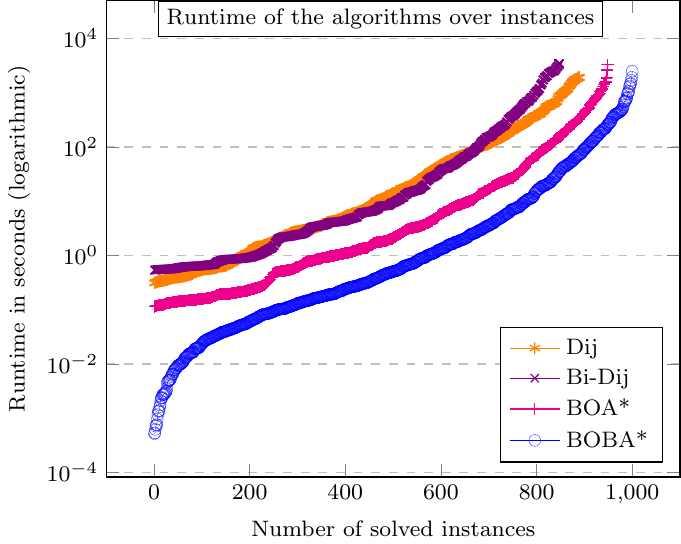}
\end{subfigure}
\begin{subfigure}{0.49\textwidth}
\includegraphics[]{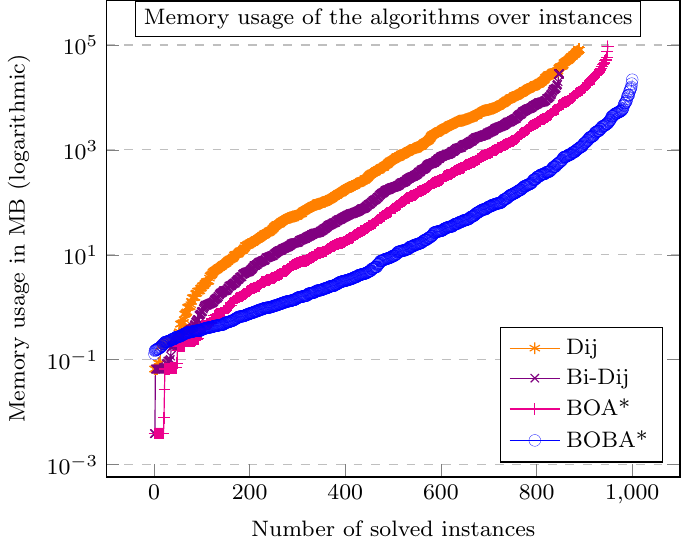}
\end{subfigure}
\vspace{-0.5 \baselineskip}
\caption{ Cactus plots of algorithms' performance. Left: Runtime. Right: Search memory usage.}
\label{fig:performance}
\end{figure}
\par
\textbf{Multi-threading:}
We investigate the impact of multi-threading in BOBA* by running the (unmodified) algorithm on a single core instead of two cores, allowing decisions on scheduling of the threads to be made by the operating system.
We compare single-core BOBA*\textsubscript{1c} with the virtual best variant BOA*\textsubscript{best} and our BOBA* with two cores in Table~\ref{table:result_time}.
The results show a slowdown of around 1.8 compared to parallel BOBA*, but it still outperforms BOA*\textsubscript{best}, solving more instances and showing an (arithmetic) average speed-up of six over all of the individual cases.
Note that this virtual best version BOA*\textsubscript{best} does not exist, and the results are based on the best timings obtained via four individual runs of the standard BOA* algorithm.\par
\textbf{Enhanced BOA*:}
To measure the contributions of our improvements to the uni-directional bi-objective search, we analyse the performance of the enhanced variant BOA*\textsubscript{enh} with the speed-up techniques above.
This variant is obtained by switching off the backward search of BOBA*.
Based on the results given in Table~\ref{table:result_time},
BOA*\textsubscript{enh} outperforms BOA*\textsubscript{best} in almost all of the cases and shows a comparable performance to BOBA*\textsubscript{1c}, solving a few more cases in the CTR map and using less memory on average.
Comparing the maximum runtime over instances, we can see that BOBA*\textsubscript{1c} is faster than BOA*\textsubscript{enh} in half of the instances (maps NY, BAY, COL, FLA and CAL).
Nonetheless, given the results in Table~\ref{table:result_time}, BOBA* is still superior to BOA*\textsubscript{enh} showing a speed-up factor of 1.5 on average.
\par
\textbf{Bucket vs. heap:}
We found BOBA* with the bucket-based \textit{Open} list around 1.8 times faster than BOBA* with heap for the same set of instances on average. Nonetheless, BOBA* with heap is still 2.2 times faster than standard heap-based BOA* (average over instances).
\section{Conclusion}
This paper introduced BOBA*, a bi-directional version of the state-of-the-art BOA* algorithm for bi-objective search.
Our new algorithm explores the graph from both (forward and backward) directions in different objective orders in parallel.
We enrich BOBA* with more efficient approaches for both the initial heuristic procedure and the solution path construction.
We also present several speed up strategies to enhance BOBA's searches in various scenarios.
Our experiments show that BOBA* outperforms the state-of-the-art algorithms in both runtime and memory use, solving all of the 1,000 benchmark cases to optimality in one hour timeout.
Furthermore, compared to BOA*, BOBA* is five times faster and needs eight times less memory on average.
Additional experiments reveal that the single-core version of BOBA* is around 1.8 times slower than the parallel version but still superior to the virtual best variant of BOA* and shows a comparable performance to BOA* enhanced with the speed-up strategies of this study.
\bibliography{References.bib}
% \onecolumn \subsection{Technical Appendix}
% \setcounter{table}{0} 
\appendix
\clearpage
\section{Backward search of BOBA*}\label{app:boba_back}
\begin{algorithm}[ht]
\footnotesize
\caption{Enhanced backward Bi-Objective A* (BOA*\textsubscript{enh}) in ($f_2,f_1$) objective ordering}
\label{alg:boa_enh_rev}
\DontPrintSemicolon
\SetAlgoLined
 \KwInput{A problem instance (Rev(G), {\bf cost}, $s_{goal}$, $s_{start}$) and heuristics (${\bf h'}$, ${\bf ub'}$, {\bf h})}
 \KwOutput{A set of cost-unique Pareto-optimal solutions}
 $Open' \leftarrow \emptyset, Sol' \leftarrow \emptyset$\ \;
 $g_2^{min}(s) \leftarrow g_1^{min}(s) \leftarrow \infty$ \ for each $s \in S$\ \;
 $x \leftarrow $ new node with $s(x) = s_{goal}$\ \;
 $ {\bf g}(x) \leftarrow (0,0) $, $ {\bf f}(x) \leftarrow (h'_1(s_{goal}),h'_2(s_{goal})) $, $parent(x) \leftarrow Null$ \;
Add $x$ to $Open'$\;
\While{$Open' \neq \emptyset$}
{
 Remove a node $x$ with the lexicographically smallest ($f_2,f_1$) values from $Open'$ \;
     \textcolor{red}{ 
       \label{alg:boa_enh_rev:termination}\nlast
        \lIf{$f_2(x) \geq g_2^{min}(s_{goal})$} 
      {  \textbf{break} }
      } 
     
      \lIf{$g_1(x) \geq g_1^{min}(s(x)) $ \textbf{or} $f_1(x) \geq g_1^{min}(s_{start}) $ \label{alg:boa_enh_rev:prune0}}  
     {\textbf{continue}} 
     
     \textcolor{red}{
     \label{alg:boa_enh_rev:tuning}\nlast
     \lIf{$g_1^{min}(s(x)) = \infty$}
    { $h_2(s(x)) \leftarrow g_2(x)$ } 
   }
    $g_1^{min}(s(x)) \leftarrow g_1(x)$ \label{alg:boa_enh_rev:min_r} \;
    \If{$s(x) = s_{goal}$} 
     { 
     \textcolor{red}{
        \label{alg:boa_enh:tie0}\nlast   $z \leftarrow $ last node in $Sol'$ \; 
        \nlast   \lIf{$(z \neq Null $ \textbf{and} $f_2(z) = f_2(x))$} { Remove $z$ from $Sol'$ } \label{alg:boa_enh:tie1}
        }
     Add $x$ to $Sol'$ \;
        \textbf{continue}
    }
     \textcolor{red}{
     \label{alg:boa_enh_rev:early_sol}\nlast
     \If{$g_1(x) + ub_1'(s(x)) < g_1^{min}(s_{start})$}
    {     \nlast  $g_1^{min}(s_{start}) \leftarrow g_1(x) + ub_1'(s(x))$ \;
         \textcolor{red}{
        \label{alg:boa_enh:tie0}\nlast   $z \leftarrow $ last node in $Sol'$ \; 
        \nlast   \lIf{$(z \neq Null $ \textbf{and} $f_2(z) = f_2(x))$} { Remove $z$ from $Sol'$ } \label{alg:boa_enh:tie1}
        }
         \nlast   Add $x$ to $Sol'$ \;
         \nlast\lIf{$h_2'(s(x))= ub_2'(s(x))$}
       { \textbf{continue} } \label{alg:boa_enh_rev:unique_path}
    }
    }
    
         \For{$all\ t \in Succ(s(x))$}
        { $y \leftarrow $ new node with $s(y) = t$ \;
             ${\bf g}(y) \leftarrow {\bf g}(x) + {\bf cost} (s(x),t)$ \;
             ${\bf f}(y) \leftarrow {\bf g}(y) + {\bf h'} (t)$ \;
             $parent(y) \leftarrow x$ \;
            \lIf{$g_1(y) \geq g_1^{min}(t) $ \textbf{or} $f_1(y) \geq g_1^{min}(s_{start}) $ \label{alg:boa_enh_rev:prune1}} 
            { \textbf{continue} }
            
            \textcolor{red}{\label{alg:boa_enh_rev:prune2}\nlast
            \lIf{$f_2 \geq g_2^{min}(s_{goal}) $ }
            { \textbf{continue} } 
            }
             Add $y$ to $Open'$\;
        }
}
\Return{$Sol'$}\; 
\end{algorithm}

\end{document}